\documentclass[11pt]{article}

\usepackage{authblk}

\usepackage{ifpdf}
\usepackage{amsmath}
\usepackage{amssymb}
\usepackage{epsf}
\usepackage{epsfig}
\usepackage{amsfonts}
\usepackage{amsthm}
\usepackage{enumerate} 
\usepackage{graphicx}
\usepackage{subfig}
\usepackage{float}
\newcommand{\Loss}{\mathop{{\rm{Loss}}}\nolimits}

\newcommand{\G}{\mathfrak G}

\ifpdf
  \usepackage[pdftex]{hyperref}
\else
  \usepackage[dvips]{hyperref}
\fi

\newcommand{\calE}{{\cal E}}

\newcommand{\calS}{{\cal S}}

\newtheorem{proposition}{Proposition}
\newtheorem{theorem}{Theorem}
\newtheorem{lemma}{Lemma}
\newtheorem{corollary}{Corollary}

\newtheorem{protocol}{Protocol}

\theoremstyle{remark}
\newtheorem{remark}{Remark}

\allowdisplaybreaks[4]

\author[1,2]{Dmitry
  Adamskiy\thanks{\href{mailto:D.Adamskiy@cs.ucl.ac.uk}{D.Adamskiy@cs.ucl.ac.uk}}}
\author[5]{Tony Bellotti\thanks{\href{mailto:A.Bellotti@imperial.ac.uk}{A.Bellotti@imperial.ac.uk}}}
\author[3]{Raisa Dzhamtyrova\thanks{\href{mailto:Raisa.Dzhamtyrova.2015@rhul.ac.uk}{Raisa.Dzhamtyrova.2015@live.rhul.ac.uk}}}
\author[3,2,4]{Yuri Kalnishkan\thanks{\href{mailto:yuri.kalnishkan@rhul.ac.uk}{Yuri.Kalnishkan@rhul.ac.uk}}}

\affil[1]{University College London,
Dept.\ of Computer Science, 66-72, Gower Street, London, WC1E 6EA
United Kingdom}

\affil[2]{Computer Learning Research Centre, Royal
    Holloway, University of London, Egham, Surrey, TW20 0EX, United Kingdom}

\affil[3]{Department of Computer Science, Royal
    Holloway, University of London, Egham, Surrey, TW20 0EX, United
    Kingdom}
\affil[4]{Laboratory of Advanced Combinatorics and Network
  Applications, Moscow Institute of Physics and Technology,
  Institutsky per., 9, Dolgoprudny, 41701, Russia} 
\affil[5]{Department of Mathematics, Imperial College London,
London, SW7 2AZ, United Kingdom}

\title{Aggregating Algorithm for Prediction of Packs}

\begin{document}

\maketitle

\begin{abstract}
This paper formulates the protocol for prediction of packs, which a
special case of prediction under delayed feedback. Under this
protocol, the learner must make a few predictions without seeing the
outcomes and then the outcomes are revealed. We develop the theory of
prediction with expert advice for packs. By applying Vovk's
Aggregating Algorithm to this problem we obtain a number of algorithms
with tight upper bounds. We carry out empirical experiments on housing
data.
\end{abstract}

\section{Introduction}

This paper deals with the on-line prediction protocol, where the
learner needs to predict outcomes $\omega_1,\omega_2\ldots$ occurring
in succession. The learner is getting the feedback along the way.

In the basic on-line prediction protocol, on step $t$ the learner
outputs a prediction $\gamma_t$ and then immediately sees the true
outcome $\omega_t$. The quality of the prediction is assessed by a
loss function $\lambda(\gamma,\omega)$ measuring the discrepancy
between the prediction and outcome or, generally speaking, quantifying
the (adverse) effect when a prediction $\gamma$ confronts the outcome
$\omega$. The performance of the learner is assessed by the cumulative
loss over $T$ trials $\Loss(T) =
\sum_{t=1}^T\lambda(\gamma_t,\omega_t)$.

In a protocol with the delayed feedback, there may be a delay getting
true $\omega$s. The learner may need to make a few predictions before
actually seeing the outcomes of past trials. We will consider a special
case of that protocol when outcomes come in packs: the learner needs
to make a few predictions, than all outcomes are revealed, and again a
few predictions need to be made.

A model problem we consider is prediction of house prices. Consider a
dataset consisting of descriptions of houses and sale prices. Suppose
that the prices come with transaction dates; it is therefore natural
to analyse this dataset in an on-line framework trying to predict
house prices on the basis of past information only.

However, let the timestamps in the dataset contain only the month of
the transaction. Every month a few sales occur and we do not know the
order. It is natural to try and work out all predicted prices for a
particular month on the basis of past months and only then to look at
the true prices. One month of transactions makes what we call a pack.

We are concerned in this paper with the problem of prediction with
expert advice. Suppose that the learner has access to predictions of
a number of experts. Before the learner makes a prediction, it can see
experts' predictions and its goal is to suffer loss close to that of
the retrospectively best expert.

The problem of prediction with expert advice is related to that of
on-line optimisation, which has been extensively studied since
\cite{zinkevich2003online}. These approaches overlap to a very large
extent and many results make sense in both the frameworks. The problem
of delayed feedback has mostly been studies within the on-line optimisation
approach, e.g., in \cite{joulani2013online,
  quanrud2015online}. However, in this paper we will stick to
the terminology and approach of prediction with expert advice going
back to \cite{maj_little_1994} and surveyed in
\cite{cbl_book_prediction}. Our starting point and the main tool is
Vovk's Aggregating Algorithm \cite{vovk_aggr, vovk_advice}, which
provides a solution optimal in a certain sense.

Our key idea is to consider a pack as a single outcome. We study
mixability of the resulting game and develop a few algorithms for
prediction of packs based on the Aggregating Algorithm. We obtain
upper bounds on their performance and discuss optimality
properties. The reason why we need different algorithms is that the
situation when the pack size varies from step to step can be addressed
in different ways leading to different bounds.

The key result of the theory of delayed feedback stating that the
regret multiplies by the magnitude of the delay
(see \cite{joulani2013online, weinberger2002delayed}) cannot be improved,
but it receives interpretation in the context of the theory of
prediction with expert advice with specific lower bounds of the
Aggregating Algorithm type. In empirical studies our new algorithms
show more stable performance than the existing algorithm based on
running parallel copies of the merging procedure.

We carry out an empirical investigation on London and Ames house
prices datasets. The experiments follow the approach of
\cite{kalnishkan2015specialist}: prediction with expert advice can
used to find relevant past information. Predictors trained on
different sections of past data can be combined in the on-line mode so
that prediction is carried out using relevant past data.

The paper is organised as follows. In Section~\ref{sect_preliminaries}
the theory of the Aggregating Algorithm is surveyed. In
Section~\ref{section_packs} we formulate the protocols for prediction
of packs and then study the mixability of resulting games. This
analysis leads to the Aggregating Algorithm for Prediction of Packs
formulated in Section~\ref{sect_algorithms}. The theory can be applied
in different ways leading to different loss bounds, hence a few
variations of the algorithm. We also describe the algorithm based on
running parallel copies of the Aggregating Algorithm: it is a
straightforward adaptation of an existing delayed feedback algorithm
to our problem. Empirical experiments are described in
Section~\ref{sect_experiments}.

As the upper bounds on the loss are based on the theory of the
Aggregating Algorithm, most of them are tight in the worst case.  As a
digression from the prediction with expert advice framework, in
Section~\ref{sect_mix_loss} we prove a self-contained lower bound for
prediction of packs in the context of the mix loss protocol of
\cite{adamskiy2016closer}.

\section{Preliminaries}
\label{sect_preliminaries}

\subsection{Prediction with Expert Advice}

In this section we formulate the classical problem of prediction with
expert advice.

A game $\G=\langle\Omega,\Gamma,\lambda\rangle$ is a triple of an {\em
  outcome space} $\Omega$, {\em prediction space} $\Gamma$, and {\em
  loss function} $\lambda: \Gamma\times\Omega\to
[0,+\infty]$. Outcomes $\omega_1,\omega_2,\ldots\in\Omega$ occur in
succession. A {\em learner} or {\em prediction strategy} outputs
predictions $\gamma_1,\gamma_2,\ldots\in\Gamma$ before seeing each
respective outcome. The learner may have access to some side
information; we will say that on each step the learner sees a {\em
  signal} $x_t$ coming from a {\em signal space} $X$.

The framework is summarised in Protocol~\ref{prot_with_signals}.

\begin{protocol}~
{\tt
\label{prot_with_signals}
\begin{tabbing} 
  \quad\=\quad\=\quad\=\quad\=\quad\kill
   FOR $t=1,2,\ldots$\\
   \>\> nature announces $x_t \in X$\\
   \>\> learner outputs $\gamma_t \in \Gamma$\\
   \>\> nature announces $\omega_t \in \Omega$\\
   \>\> learner suffers loss  $\lambda(\gamma_t,\omega_t)$\\
   ENDFOR
\end{tabbing}
}
\end{protocol}

Over $T$ trials the learner $\calS$ suffers the cumulative loss $
\Loss_T=\Loss_T(\calS) =\sum_{t=1}^T\lambda(\gamma_t,\omega_t)$.

In this paper we assume a full information environment. The learner
knows $\Omega$, $\Gamma$, and $\lambda$. It sees all $\omega_t$ as
they become available. On the other hand, we make no assumptions on
the mechanism generating $\omega_t$ and will be interested in
worst-case guarantees for the loss.

Now let $\{\calE_\theta\mid \theta\in\Theta\}$ be a set of learners
working according to Protocol~\ref{prot_with_signals} and parametrised
by $\theta\in\Theta$. We will refer to these learners as {\em experts}
and to the set as the {\em pool of experts}. If the pool is finite and
$|\Theta|=N$, we will refer to experts as
$\calE_1,\calE_2,\ldots,\calE_N$.  Suppose that on each turn, their
predictions are made available to a learner $\calS$ as a special kind
of side information. The learner then works according to the following
protocol.

\begin{protocol}~
  {\tt
\label{prot_experts}
\begin{tabbing}
\quad\=\quad\=\quad\=\quad=\quad\kill
 FOR $t=1,2,\ldots$\\
 \>\> experts $\calE_\theta$, $\theta\in\Theta$ announce predictions $\gamma_t^\theta\in \Gamma$\\
 \>\> learner outputs $\gamma_t \in \Gamma$\\
 \>\> nature announces $\omega_t \in \Omega$\\
 \>\> each expert $\theta\in\Theta$ suffers loss $\lambda(\gamma_t^\theta,\omega_t)$\\
 \>\> learner suffers loss  $\lambda(\gamma_t,\omega_t)$\\
ENDFOR
\end{tabbing}
}
\end{protocol}

The goal of the learner in this setup is to suffer loss close to the
best expert in retrospect. We look for {\em merging strategies} giving
guarantees of the type $\Loss_T(\calS) \lesssim \Loss_T(\calE_\theta)$
for all $\theta\in\Theta$, all sequences of outcomes, and as many $T$
as possible.

The merging strategies we are interested in are computable in some
natural sense; we will not make exact statements about computability
though. We do not impose any restrictions on experts. In what follows,
the reader may substitute the clause `for all predictions
$\gamma_t^\theta$ appearing in Protocol~\ref{prot_experts}' for the
more intuitive clause `for all experts'.

\subsection{Aggregating Algorithm}

In this section we present Vovk's Aggregating Algorithm (AA) after
\cite{vovk_aggr, vovk_advice}. In this paper we restrict ourselves to
finite pools of experts, but AA can be straightforwardly extended to
countable pools (by considering infinite sums) and even larger pools
(by replacing sums with integrals).

The algorithm takes a parameter $\eta>0$ called the {\em learning
  rate} and a prior distribution $p^1,p^2,\ldots,p^N$ ($p^n\ge 0$ and
$\sum_{i=1}^Np^n=1$) on experts 
$\calE_1,\calE_2,\ldots,\calE_N$.

A constant $C>0$ is {\em admissible} for a learning rate $\eta>0$ if
for every $N=1,2,\ldots$, every sets of predictions
$\gamma^1,\gamma^2,\ldots,\gamma^N$, and every distributions
$p^1,p^2,\ldots,p^N$ (such that $p^n\ge 0$ and $\sum_{i=1}^Np^n=1$)
there is $\gamma\in\Gamma$ ensuring for all outcomes $\omega\in\Omega$
the inequality
\begin{equation}
\label{eq_mixability}
\lambda(\gamma,\omega)\le
-\frac{C}{\eta}\ln\sum_{n=1}^Np^ne^{-\eta\lambda(\gamma^n,\omega)} \enspace.
\end{equation}

The {\em mixability constant} $C_\eta$ is the infimum of all $C>0$
admissible for $\eta$. This infumum is usually achieved.  For example,
it is achieved for all $\eta>0$ whenever $\Gamma$ is compact and
$e^{-\lambda(\gamma,\omega)}$ is continuous\footnote{Or
  $\lambda(\gamma,\omega)$ is continuous w.r.t.\ the extended topology
  of $[0,+\infty]$.} in $\gamma$.

The AA works as follows. It takes as parameters a prior distribution 
$p^1,p^2,\ldots,p^N$ (such that $p^n\ge 0$ and $\sum_{i=1}^Np^n=1$), a
learning rate $\eta>0$ and a constant $C$ admissible for $\eta$.

\begin{protocol}~
  {\tt
\label{prot_AA}
\begin{tabbing}
  \quad\=\quad\=\quad\=\quad\kill
  (1) initialise weights $w_0^n=p^n$, $n=1,2,\ldots,N$ \\   
  (2) FOR $t=1,2,\dots$\\
  (3) \>\>\> read the experts' predictions $\gamma_t^n$, $n=1,2,\ldots,N$ \\ 
  (4) \>\>\> normalise the weights $p_{t-1}^n=w^n_{t-1}/\sum_{i=1}^Nw^i_{t-1}$\\
  (5) \>\>\> output $\gamma_t\in\Gamma$ satisfying for all
  $\omega\in\Omega$ the inequality\\
  \>\>\>$\lambda(\gamma_t,\omega)\le 
-\frac{C}{\eta}\ln\sum_{n=1}^Np_{t-1}^ne^{-\eta\lambda(\gamma_t^n,\omega)}$\\
  (6) \>\>\> observe the outcome $\omega_t$ \\
  (7) \>\>\> update the experts' weights
  $w_{t}^n=w_{t-1}^ne^{-\eta\lambda(\gamma_t^n,\omega_t)}$,\\ \>\> $n=1,2,\ldots,N$ \\
  (8) END FOR
\end{tabbing}
  }
\end{protocol}

\begin{proposition}
\label{prop_AA}
Let $C$ be admissible for $\eta>0$. Then for every $N=1,2,\ldots$,
the loss of a learner $\calS$ using the AA with $\eta$ and a prior
distribution $p^1,p^2,\ldots,p^N$ satisfies
\begin{equation}
\label{AA_guarantee}
\Loss_T(\calS)\le C\Loss_T(\calE_n)+\frac{C}{\eta}\ln\frac{1}{p^n}
\end{equation}
for every expert $\calE_n$, $n=1,2,\ldots,N$, all time horizons
$T=1,2,\ldots$, and all outputs made by the nature.  
\end{proposition}

\begin{proof}[Proof Sketch]
Inequality~\ref{eq_mixability} can be rewritten as
$$
e^{-\eta\lambda(\gamma,\omega)/C}\ge
\sum_{n=1}^Np^ne^{-\eta\lambda(\gamma^n,\omega)}\enspace.
$$
One can check by induction that the equality
$$
e^{-\eta\Loss_t(\calS)/C}\ge \sum_{n=1}^Np^ne^{-\eta\Loss_t(\calE_n)}
$$
holds for all $t=1,2,\ldots$. Dropping all terms but one on the
right-hand side yields the desired inequality.
\end{proof}

The importance of the AA follows from the results of
\cite{vovk_advice}. Under some mild regularity assumptions on the game
and assuming the uniform initial distribution, it can be shown that
the constants in \ref{AA_guarantee} are optimal. If any merging
strategy achieves the guarantee
$$
\Loss_T(\calS)\le C\Loss_T(\calE_n)+A\ln N
$$
for all experts $\calE_1,\calE_2,\ldots,\calE_N$, $N=1,2,\ldots$, all
time horizons $T$, and all outcomes, then the AA with the uniform
prior distribution $p^n=1/N$ and some $\eta>0$ provides the guarantee
with the same or lower $C$ and $A$.

\section{Prediction of Packs}
\label{section_packs}

\subsection{Protocol}

Consider the following extension of Protocol~\ref{prot_with_signals}.

\begin{protocol}~
{\tt
\label{prot_with_signals_pack}
\begin{tabbing} 
  \quad\=\quad\=\quad\=\quad\=\quad\kill
   FOR $t=1,2,\ldots$\\
   \>\> nature announces $x_{t,k} \in X$, $k=1,2,\ldots,K_t$\\
   \>\> learner outputs $\gamma_{t,k} \in \Gamma$,  $k=1,2,\ldots,K_t$\\
   \>\> nature announces $\omega_{t,k} \in \Omega$,  $k=1,2,\ldots,K_t$\\
   \>\> learner suffers losses  $\lambda(\gamma_{t,k},\omega_t)$,  $k=1,2,\ldots,K_t$\\
   ENDFOR
\end{tabbing}
}
\end{protocol}

In summary, at every trial $t$ the learner needs to make $K_t$
predictions rather than one.

Suppose that the learner may draw help from experts. We can extend
Protocol~\ref{prot_experts} as follows.

\begin{protocol}~
  {\tt
\label{prot_pack}
\begin{tabbing}
\quad\=\quad\=\quad\=\quad=\quad\kill
 FOR $t=1,2,\ldots$\\
 \>\> each expert $\calE_\theta$, $\theta\in\Theta$ announces\\
 \>\>\> predictions
$\gamma_{t,k}^\theta\in \Gamma$, $k=1,2,\ldots,K_t$\\
 \>\> learner outputs predictions $\gamma_{t,k} \in \Gamma$, $k=1,2,\ldots,K_t$ \\
 \>\> nature announces $\omega_{t,k} \in \Omega$, $k=1,2,\ldots,K_t$\\
 \>\> each expert $\theta\in\Theta$ suffers loss $\sum_{k=1}^{K_t}\lambda(\gamma_{t,k}^\theta,\omega_{t,k})$\\
 \>\> learner suffers loss  $\sum_{k=1}^{K_t}\lambda(\gamma_{t,k},\omega_{t,k})$\\
ENDFOR
\end{tabbing}
}
\end{protocol}

There can be subtle variations of this protocol. Instead of getting
all $K_t$ predictions from each expert at once, the learner may be
getting predictions for each outcome one by one and making its own
before seeing the next set of experts' predictions. For most of our
analysis this does not matter, as we will see later. The learner may
have to work on each `pack' of experts' predictions sequentially
without even knowing its size in advance. The only thing that make a
difference is that the outcomes come in one go after the learner has
finished predicting the pack.

\subsection{Mixability}

For a game $\G = \langle\Omega,\Gamma,\lambda\rangle$ and a positive
integer $K$ consider the game $\G^K$ with the outcome and prediction
space given by the Cartesian products $\Omega^K$ and $\Gamma^K$ and
the the loss function
$\lambda^{(K)}((\gamma_1,\gamma_2,\ldots,\gamma_K),(\omega_1,
\omega_2,\ldots,\omega_K))=\sum_{k=1}^K\lambda(\gamma_K,\omega_K)$. What
are the mixability constants for this game? Let $C_\eta$ be the
constants for $\G$ and $C_\eta^{(K)}$ be the constants for
$\G^{(K)}$.


The following lemma provides an upper bound for $C_\eta^{(K)}$.

\begin{lemma}
\label{lem_upper}
  For every game $\G$ we have $C^{(K)}_{\eta/K}\le C_\eta$.
\end{lemma}

\begin{proof}
Take $N$ predictions in the game $\G^K$, $\gamma^1 =
(\gamma^1_1,\gamma^1_2,\ldots,\gamma^1_K)$, \ldots, $\gamma^N =
(\gamma^N_1,\gamma^N_2,\ldots,\gamma^N_K)$ and weights
$p^1,p^2,\ldots,p^N$. Let $\gamma_1,\gamma_2,\ldots,\gamma_K\in\Gamma$
be some predictions satisfying
$$
e^{-\eta\lambda(\gamma_k,\omega_k)/C}\ge
\sum_{n=1}^Np^ne^{-\eta\lambda(\gamma^n_k,\omega_k)}
$$
for every $\omega_k\in\Omega$. Multiplying these inequalities yields
$$
e^{-\eta\sum_{k=1}^K\lambda(\gamma_k,\omega_k)/C}\ge \prod_{k=1}^K\sum_{n=1}^Np^ne^{-\eta\lambda(\gamma^n_k,\omega_k)}\enspace.
$$ We will now apply the generalised H\"older inequality. On measure
spaces, the inequality states that $\|\prod_{k=1}^Kf_k\|_r\le
\prod_{k=1}^K\|f_k\|_{r_k}$, where $\sum_{k=1}^K1/r_k=1/r$ (this
follows from the version of the inequality in Section~9.3 of
\cite{loeve_vol1} by induction). Interpreting a vector $x =
(x_1,x_2,\ldots,x_N)$ as a function on a discrete space
$\{1,2,\ldots,N\}$ and introducing on this space the measure $\mu(n) =
p^n$, we obtain
$$
\left(\sum_{n=1}^Np^n\left|\prod_{k=1}^Kx^k_n\right|^r\right)^{1/r} \le
\prod_{k=1}^K\left(\sum_{n=1}^Np^n\left|x^k_n\right|^{r_k}\right)^{1/r_k}\enspace.
$$
Letting $r_k=1$ and $r=1/K$ we get
\begin{align*}
e^{-\eta\sum_{k=1}^K\lambda(\gamma_k,\omega_k)/C} &\ge
\prod_{k=1}^K\sum_{n=1}^Np^ne^{-\eta\lambda(\gamma^n_k,\omega_k)}\\
&\ge
\left(\sum_{n=1}^Np^ne^{-\sum_{k=1}^K\eta\lambda(\gamma^n_k,\omega_k)/K}\right)^{K}
\enspace.
\end{align*}
Raising the resulting inequality to the power $1/K$ completes the proof.
\end{proof}

\begin{remark}
Note that the proof of the lemma offers a constructive way of
solving~(\ref{eq_mixability}) for $\G^{K}$ provided we know how to
solve~(\ref{eq_mixability}) for $\G$. Namely, to
solve~(\ref{eq_mixability}) for $\G^{K}$ with the learning rate
$\eta/K$, we solve $K$ systems for $\G$ with the learning rate $\eta$.
\end{remark}

In order to get a lower bound for $C_\eta^{(K)}$, we need the
following concepts.

A {\em generalised prediction} w.r.t.\ a game $\G$ is a function from
$\Omega$ to $[0,+\infty]$. Every prediction $\gamma\in\Gamma$
specifies a generalised prediction by $\lambda(\gamma,\cdot)$, hence
the name.

A {\em superprediction} is a generalised prediction minorised by some
prediction, i.e., a superprediction is a function $f:\Omega\to
[0,+\infty]$ such that for some $\gamma\in\Gamma$ we have
$f(\omega)\ge \lambda(\gamma,\omega)$ for all $\omega\in\Omega$. The
shape of the set of superpredictions plays a crucial role in
determining $C_\eta$.

For a wide class of games the following implication holds. If the game
is mixable (i.e., $C_\eta=1$ for some $\eta>0$), then its set of
superpredictions is convex (Lemma~7 in~\cite{me_lossleg} essentially
proves this for games with finite sets of outcomes). 

\begin{lemma}
\label{lem_lower}
For every game $\G$ with a convex set of superpredictions we have
$C^{(K)}_{\eta/K}\ge C_\eta$.
\end{lemma}

\begin{proof}
Let for every  $\gamma^1 =
(\gamma^1_1,\gamma^1_2,\ldots,\gamma^1_K)$, \ldots,  $\gamma^N =
(\gamma^N_1,\gamma^N_2,\ldots,\gamma^N_K)$ and weights
$p^1,p^2,\ldots,p^N$ be
$\gamma_1,\gamma_2,\ldots,\gamma_K\in\Gamma$ such that
$$
\sum_{k=1}^K\lambda(\gamma_k,\omega_k)\le
-\frac{C}{\eta/K}\ln\sum_{n=1}^Np^ne^{-\eta\sum_{k=1}^K\lambda(\gamma_k,\omega_k)/K}
$$
for all $\omega_1,\omega_2,\ldots,\omega_K\in\Omega$.

Given $N$ predictions $\gamma^*_1,\gamma^*_2,\ldots,\gamma^*_K$,
consider $\gamma^n=(\gamma^*_n,\ldots,\gamma^*_n)$ ($K$ times),
$n=1,2,\ldots,N$. There is an array
$\gamma_1,\gamma_2,\ldots,\gamma_K\in\Gamma$ satisfying
$$
\frac{1}{K}\sum_{k=1}^K\lambda(\gamma_k,\omega)\le
-\frac{C}{\eta}\ln\sum_{n=1}^Np^ne^{-\eta\lambda(\gamma^*_k,\omega)}
$$
for all $\omega\in\Omega$ (we let $\omega_1=\omega_2=\ldots=\omega_K=\omega$).

The problem is that $\gamma_k$ do not have to be equal and do not
collate to one prediction. However,
$\sum_{k=1}^K\lambda(\gamma_k,\omega)/K$ is a convex combination of
superpredictions w.r.t.\ $\G$. Since the set of superpredictions is
convex, this expression is a superprediction and there is
$\gamma\in\Gamma$ such that $\lambda(\gamma,\omega)\le
\sum_{k=1}^K\lambda(\gamma_k,\omega)/K$ for all $\omega\in\Omega$.
\end{proof}

We get the following theorem.

\begin{theorem}
For a game $\G$ with a convex set of superprediction, any positive
integer $K$ and learing rate $\eta>0$ we have $C^{(K)}_{\eta/K}=C_\eta$.
\end{theorem}

We need to make a simple observation on the behaviour of
$C^{(K_1)}_{\eta/{K_2}}$ for $K_1\le K_2$.

\begin{lemma}
For every game $\G$ the value of $C_\eta$ is non-decreasing in $\eta$.
\end{lemma}

\begin{proof}
Suppose that
\begin{equation*}
e^{-\eta_1\lambda(\gamma,\omega)/C}\ge
\sum_{n=1}^Np^ne^{-\eta_1\lambda(\gamma^n,\omega)}
\end{equation*}
and $\eta_2\le\eta_1$. Raising the inequality to the power
$\eta_2/\eta_1\le 1$ and using Jensen's inequality yields 
\begin{align*}
e^{-\eta_2\lambda(\gamma,\omega)/C}&\ge
\left(\sum_{n=1}^Np^ne^{-\eta_1\lambda(\gamma^n,\omega)}\right)^{\eta_2/\eta_1}\\
&\ge
\sum_{n=1}^Np^ne^{-\eta_2\lambda(\gamma^n,\omega)}\enspace.
\end{align*}
\end{proof}

\begin{remark}
The proof is again constructive in the following sense. If we know how
to solve~(\ref{eq_mixability}) for $\G$ with a learning rate $\eta_1$
and an admissible $C$, we can solve~(\ref{eq_mixability}) for
$\eta_2\le\eta_1$ and the same $C$.
\end{remark}

\begin{corollary}
\label{cor_larger_K}
For every game $\G$ and positive integers $K_1\le K_2$, we have
$C^{(K_1)}_{\eta/K_2}\le C^{(K_1)}_{\eta/K_1}$.
\end{corollary}

\begin{remark}
\label{rem_larger_K}
Suppose we play the game $\G^{(K_1)}$ but have to use the learning
rate $\eta/K_2$ with $C$ admissible for $\G$ with $\eta$. To
solve~(\ref{eq_mixability}), we can take $K_1$ solutions
for~(\ref{eq_mixability}) for $\G$ with the learning rate $\eta$.
\end{remark}

\section{Algorithms for Prediction of Packs}
\label{sect_algorithms}

\subsection{Prediction with Plain Bounds}
\label{sect_plain_bounds}

Suppose that in Protocol~\ref{prot_pack} the sizes of all packs are
equal: $K_1=K_2=\ldots=K$ and the number $K$ is known in
advance. The proof of Lemma~\ref{lem_upper} suggests the following
merging strategy, which we will call {\em Aggregating Algorithm for
  Equal Packs} (AAP-e).

\begin{protocol}~
  {\tt
\label{prot_AAP-e}
\begin{tabbing}
  (10)\=\quad\=\quad\=\quad\=\quad\=\quad\=\quad\kill
  (1) initialise weights $w_0^n=p^n$, $n=1,2,\ldots,N$ \\   
  (2) FOR $t=1,2,\dots$\\
  (3) \>\>\> normalise the weights $p_{t-1}^n=w^n_{t-1}/\sum_{i=1}^Nw^i_{t-1}$\\ 
  (4) \>\>\> FOR $k=1,2,\ldots,K$\\
  (5) \>\>\>\>\> read the experts' predictions $\gamma_{t,k}^n$, $n=1,2,\ldots,N$ \\ 
  (6) \>\>\>\>\> output $\gamma_{t,k}\in\Gamma$ satisfying for all
  $\omega\in\Omega$ the \\
  \>\>\>\>\> inequality $\lambda(\gamma_{t,k},\omega)\le 
  -\frac{C}{\eta}\ln\sum_{n=1}^Np_{t-1}^ne^{-\eta\lambda(\gamma_{t,k}^n,\omega)}$\\
  (7) \>\>\> ENDFOR\\
  (8) \>\>\> observe the outcomes $\omega_{t,k}$, $k=1,2,\ldots,K$ \\
  (9) \>\>\> update the experts' weights
  $w_{t}^n=w_{t-1}^ne^{-\eta\sum_{k=1}^K\lambda(\gamma_{t,k}^n,\omega_{t,k})/K}$,\\
  \>\>\> $n=1,2,\ldots,N$ \\
  (10) END FOR
\end{tabbing}
  }
\end{protocol}

This algorithm essentially applies AA to $\G^K$ with the learning rate
$\eta/K$.

If we extend the meaning of $\Loss$ for a strategy $\calS$ working in
the environment specified by Protocol~\ref{prot_with_signals_pack} as
follows:
$$
\Loss_T(\calS) = \sum_{t=1}^T\sum_{k=1}^{K_t}\lambda(\gamma_{t,k},\omega_{t,k})\enspace,
$$
we get the following theorem.

\begin{theorem}
If $C$ is admissible for $\G$ with the learning rate $\eta$, then
the learner following AAP-e suffers loss satisfying
$$
\Loss_T(\calS)\le
C\Loss_{\calE_n}(\calS)+\frac{KC}{\eta}\ln\frac{1}{p^n}
$$ for all outcomes and experts' predictions as long as the pack size
is $K$.
\end{theorem}

Lemma~\ref{lem_lower} shows that the constants in this bound cannot be
improved for equal weights provided $\G$ has a convex set of
superpredictions (and $\G^K$ satisfies the conditions of the
optimality of AA).

Now suppose that $K_t$ differ. To begin with, suppose that we know $K$
upper bounding all $K_t$. Consider the following algorithm, 
{\em Aggregating Algorithm for Packs with the Known Maximum} (AAP-max).

\begin{protocol}~
  {\tt
\label{prot_AAP-max}
\begin{tabbing}
  (10)\=\quad\=\quad\=\quad\=\quad\=\quad\=\quad\kill
  (1) initialise weights $w_0^n=p^n$, $n=1,2,\ldots,N$ \\   
  (2) FOR $t=1,2,\dots$\\
  (3) \>\>\> normalise the weights $p_{t-1}^n=w^n_{t-1}/\sum_{i=1}^Nw^i_{t-1}$\\ 
  (4) \>\>\> FOR $k=1,2,\ldots,K_t$\\
  (5) \>\>\>\>\> read the experts' predictions $\gamma_{t,k}^n$, $n=1,2,\ldots,N$ \\ 
  (6) \>\>\>\>\> output $\gamma_{t,k}\in\Gamma$ satisfying for all
  $\omega\in\Omega$ the \\
  \>\>\>\>\> inequality $\lambda(\gamma_{t,k},\omega)\le 
  -\frac{C}{\eta}\ln\sum_{n=1}^Np_{t-1}^ne^{-\eta\lambda(\gamma_{t,k}^n,\omega)}$\\
  (7) \>\>\> ENDFOR\\
  (8) \>\>\> observe the outcomes $\omega_{t,k}$, $k=1,2,\ldots,K_t$ \\
  (9) \>\>\> update the experts' weights
  $w_{t}^n=w_{t-1}^ne^{-\eta\sum_{k=1}^K\lambda(\gamma_{t,k}^n,\omega_{t,k})/K}$,\\
  \>\>\> $n=1,2,\ldots,N$ \\
  (10) END FOR
\end{tabbing}
  }
\end{protocol}

The essential point here is step (9): we divide by the maximum $K$.

Corollary~\ref{cor_larger_K} and Remark~\ref{rem_larger_K} imply the
following result.

\begin{theorem}
If $C$ is admissible for $\G$ with the learning rate $\eta$, then
the learner following AAP-m suffers loss satisfying
$$
\Loss_T(\calS)\le
C\Loss_{\calE_n}(\calS)+\frac{KC}{\eta}\ln\frac{1}{p^n}
$$
for all outcomes and experts' predictions as long as the pack size
does not exceed $K$.
\end{theorem}

Clearly, the constants in this bound cannot be improved in the same
sense as above because of the case where all packs have the maximum
size $K$. However, the algorithm clearly uses a suboptimal learning
rate for steps with $K_t<K$. We will address this later.

Now consider the case where $K$ is not known in advance. A simple
trick allows one to handle this inconvenience. Consider the following
algorithm, {\em Aggregating Algorithm for Packs with an Unknown
  Maximum} (AAP-incremental).

\begin{protocol}~
  {\tt
\label{prot_AAP-i}
\begin{tabbing}
  (10)\=\quad\=\quad\=\quad\=\quad\=\quad\=\quad\kill
  (1) initialise losses $L_0^n=0$, $n=1,2,\ldots,N$ \\
  (2) initialise $K^\mathrm{max}_0 = 1$\\
  (3) set weights to  $w_0^n=p^n$, $n=1,2,\ldots,N$ \\   
  (4) FOR $t=1,2,\dots$\\
  (5) \>\>\> normalise the weights $p_{t-1}^n=w^n_{t-1}/\sum_{i=1}^Nw^i_{t-1}$\\ 
  (6) \>\>\> FOR $k=1,2,\ldots,K_t$\\
  (7) \>\>\>\>\> read the experts' predictions $\gamma_{t,k}^n$, $n=1,2,\ldots,N$ \\ 
  (8) \>\>\>\>\> output $\gamma_{t,k}\in\Gamma$ satisfying for all
  $\omega\in\Omega$ the \\
  \>\>\>\>\> inequality $\lambda(\gamma_{t,k},\omega)\le 
  -\frac{C}{\eta}\ln\sum_{n=1}^Np_{t-1}^ne^{-\eta\lambda(\gamma_{t,k}^n,\omega)}$\\
  (9) \>\>\> ENDFOR\\
  (10) \>\>\> observe the outcomes $\omega_{t,k}$, $k=1,2,\ldots,K_t$ \\
  (11) \>\>\> update the losses $L_t^n =
  L_{t-1}^n+\sum_{k=1}^{K_t}\lambda(\gamma_{t,k}^n, \omega_{t,k})$,
  $n=1,2,\ldots,N$\\
  (12) \>\>\> update $K^\mathrm{max}_t = \max(K^\mathrm{max}_{t-1},K_t)$\\
  (13) \>\>\> update the experts' weights
  $w_{t}^n=p^ne^{-\eta L^n_t/K^\mathrm{max}_t}$, $n=1,2,\ldots,N$ \\
  (14) END FOR
\end{tabbing}
  }
\end{protocol}

\begin{theorem}
  \label{theorem_incremental}
If $C$ is admissible for $\G$ with the learning rate $\eta$, then
the learner following AAP-incremental suffers loss satisfying
$$
\Loss_T(\calS)\le
C\Loss_{\calE_n}(\calS)+\frac{KC}{\eta}\ln\frac{1}{p^n}\enspace,
$$ where $K$ is the maximum pack size over $T$ trials, for all
outcomes and experts' predictions
\end{theorem}

\begin{proof}
We will show by induction over time that the inequality 
\begin{equation}
\label{ineq_induction}
e^{-\eta\Loss_t(\calS)/(CK)}\ge \sum_{n=1}^Np^ne^{-\eta\Loss_t(\calE_n)/K}
\end{equation}
holds with $K$ equal to the maximum pack size over the first $t$
trials.

Suppose that the inequality holds on trial $t$. If on trial $t+1$ the
pack size does not exceed $K$, we essentially use AA with the
learning rate $\eta/K$ and maintain the inequality.

If the pack size changes to $K'>K$, we change the learning rate to
$\eta/K'$.  Raising~(\ref{ineq_induction}) to the power $K/K'\le 1$
and applying Jensen's inequality yields
\begin{equation}
e^{-\eta\Loss_t(\calS)/(CK')}\ge \sum_{n=1}^Np^ne^{-\eta\Loss_t(\calE_n)/K'}\enspace.
\end{equation}
Over the next trial, the inequality stays.
\end{proof}

\subsection{Prediction with Bounds on Pack Averages}

The bounds in Section~\ref{sect_plain_bounds} are optimal if all
packs are of the same size. On packs of smaller size there is some
slack.

In this section we present an algorithm that fixes this
problem. However, it results in an unusual kind of bound.

Consider the following algorithm, 
{\em Aggregating Algorithm for Pack Averages} (AAP-current).

\begin{protocol}~
  {\tt
\label{prot_AAP-current}
\begin{tabbing}
  (10)\=\quad\=\quad\=\quad\=\quad\=\quad\=\quad\kill
  (1) initialise weights $w_0^n=p^n$, $n=1,2,\ldots,N$ \\   
  (2) FOR $t=1,2,\dots$\\
  (3) \>\>\> normalise the weights $p_{t-1}^n=w^n_{t-1}/\sum_{i=1}^Nw^i_{t-1}$\\ 
  (4) \>\>\> FOR $k=1,2,\ldots,K_t$\\
  (5) \>\>\>\>\> read the experts' predictions $\gamma_{t,k}^n$, $n=1,2,\ldots,N$ \\ 
  (6) \>\>\>\>\> output $\gamma_{t,k}\in\Gamma$ satisfying for all
  $\omega\in\Omega$ the \\
  \>\>\>\>\> inequality $\lambda(\gamma_{t,k},\omega)\le 
  -\frac{C}{\eta}\ln\sum_{n=1}^Np_{t-1}^ne^{-\eta\lambda(\gamma_{t,k}^n,\omega)}$\\
  (7) \>\>\> ENDFOR\\
  (8) \>\>\> observe the outcomes $\omega_{t,k}$, $k=1,2,\ldots,K_t$ \\
  (9) \>\>\> update the experts' weights
  $w_{t}^n=w_{t-1}^ne^{-\eta\sum_{k=1}^K\lambda(\gamma_{t,k}^n,\omega_{t,k})/K_t}$,\\
  \>\>\> $n=1,2,\ldots,N$ \\
  (10) END FOR
\end{tabbing}
  }
\end{protocol}

In line (9) we divide by the size of the current pack.

In order to obtain an upper bound on the loss of this algorithm, we
need a simple fact. Let $\G=\langle\Omega,\Gamma,\lambda\rangle$ be a
game and $a>0$ a constant. Let $a\G$ be the game with the same outcome
and prediction spaces and the loss function given by
$a\lambda(\gamma,\omega)$. Let $C_{a,\eta}$ be the mixability
constants for $a\G$.

\begin{lemma}
For every $a,\eta>0$ we have $C_{a,\eta}=C_{\eta a}$.
\end{lemma}

For the game $\G^{(K)}/K$ the lemma implies that $C^{(K)}_{1/K,\eta}=
C^{(K)}_{\eta/K}\le C_\eta$. Thus irrespective of $K>0$, the game
$\G^{(K)}/K$ allows all admissible constants of $\G$.

For a strategy $\calS$ working in the environment specified by
Protocol~\ref{prot_with_signals_pack} let 
$$
\Loss_T^\mathrm{average}(\calS) =
\sum_{t=1}^T\frac{\sum_{k=1}^{K_t}\lambda(\gamma_{t,k},\omega_{t,k})}{K_t}\enspace.
$$
We get the following theorem.

\begin{theorem}
If $C$ is admissible for $\G$ with the learning rate $\eta$, then
the learner following AAP-current suffers loss satisfying
$$ \Loss_T^\mathrm{average}(\calS)\le
C\Loss_{\calE_n}^\mathrm{average}(\calS)+\frac{C}{\eta}\ln\frac{1}{p^n}
$$
for all outcomes and experts' predictions.
\end{theorem}

This bound is very tight because on every step the algorithm uses the
right learning rate. It implies the following looser bound inferior to
those from Section~\ref{sect_plain_bounds}.

\begin{corollary}
  \label{cor_current}
If $C$ is admissible for $\G$ with the learning rate $\eta$, then
the learner following AAP-current suffers loss satisfying
$$ \Loss_T(\calS)\le
\frac{K_\mathrm{max}}{K_\mathrm{min}}C\Loss_{\calE_n}(\calS)+\frac{CK_\mathrm{max}}{\eta}\ln\frac{1}{p^n}
$$
for all outcomes and experts' predictions
\end{corollary}

\subsection{Parallel Copies}

In this section, we describe an existing algorithm based on running
parallel copies of the merging strategy. We will call it Parallel
Copies. It is essentially the BOLD from~\cite{joulani2013online}.

The algorithm applies to a slightly more general protocol with delayed
feedback. Under this protocol, on every step the learner gets just one
round of predictions from each expert and must produce one
prediction. However, the outcomes become available later. The delay is
the number of trials between making a prediction and obtaining
outcomes. In standard Protocol~\ref{prot_experts} the delay is always
one. Prediction of packs of size not exceeding $K$ can be
considered as prediction with delays not exceeding $K$.

The algorithm is as follows. We fix a base merging algorithm working
according to Protocol~\ref{prot_experts}. We will maintain an array of
base algorithms. An algorithm in the array is ready to predict if it
knows outcomes for all predictions it has made; otherwise it is
blocked.

At each trial, when a new round of experts' predictions arrive, we
pick a ready algorithm from the array (say, one with the lowest
number) and give the experts' predictions to it. It produces an output,
which we pass on, and the algorithm becomes blocked until the outcome
for that trials arrive. If all algorithms are currently blocked, we
add a new copy of the base algorithm to the array.

Suppose that we are playing a game $\G$ and $C$ is admissible for $\G$
with a learning rate $\eta$. For the base algorithm take AA with $C$,
$\eta$ and initial weights $p^1,p^2,\ldots,p^N$. If the delay never
exceeds $D$, we never need more than $D$ algorithms in the array and
each of them suffers loss satisfying
Proposition~\ref{prop_AA}. Summing the bounds up, we get that the loss
of $\calS$ using this strategy satisfies
\begin{equation}
\label{AA_delay}
\Loss_T(\calS)\le C\Loss_T(\calE_n)+\frac{CD}{\eta}\ln\frac{1}{p^n}
\end{equation}
for every expert $\calE_n$.

The value of $D$ does not need to be known in advance; we can always
expand the array as the delay increases.

Note that that the protocol with delays is more general than the
protocol of packs. On the other hand, for the parallel copies of
the algorithms the order in the pack matters. This cannot be seen
from~\ref{AA_delay}, but obviously happens: it is important which
example is picked by each copy.

\section{A Mix Loss Lower Bound}
\label{sect_mix_loss}

The loss bounds in the theorems formulated above are often tight due
to the optimality of the Aggregating Algorithm. The tightness was
discussed after the corresponding results.

In this section we present a self-contained lower bound formulated for
the mix loss protocol of \cite{adamskiy2016closer}. The proof sheds
some more light on the extra term in the bound. 

The mix loss protocol covers a number of learning settings including
prediction with a mixable loss function; see Section~2 of
\cite{adamskiy2016closer} for a discussion.  Consider the following
protocol porting mix loss Protocol~1 from \cite{adamskiy2016closer} to
prediction of packs.

\begin{protocol}~
  {\tt
\label{prot_pack_mix_loss}
\begin{tabbing}
\quad\=\quad\=\quad\=\quad=\quad\kill
 FOR $t=1,2,\ldots$\\
 \>\> nature announces $K_t$\\
 \>\> learner outputs $K_t$ arrays of $N$ probabilities\\
 \>\>\> $p^1_{t,k},p^2_{t,k},\ldots,p^N_{t,k}$, $k=1,2,\ldots,K_t$,
 such that\\
 \>\>\> $p^n_{t,k}\in [0,1]$ for all $n$ and $k$ and $\sum_{n=1}^Np^n_{t,k}=1$ for all $k$\\
 \>\> nature announces losses
 $\ell^n_{t,1},\ell^n_{t,2},\ldots,\ell^n_{t,K_t}\in (-\infty,+\infty]$\\
 \>\> learner suffers loss $\ell_t =
-\sum_{k=1}^{K_t}\ln\sum_{n=1}^Np^n_{t,k}e^{-\ell_{t,k}}$\\
  ENDFOR
\end{tabbing}
}
\end{protocol}

The total loss of the learner over $T$ steps is
$L_T=\sum_{t=1}^T\ell_t$. It should compare well against
$L_T^n=\sum_{t=1}^T\ell^n_t$, where $\ell^n_t = \sum_{k=1}^{K_t}
\ell^n_{t,k}$. The values of $L^n_T$ are the counterparts of experts'
total losses. We shall propose a course of action for the nature
leading to a high value of the {\em regret}
$L_T-\min_{n=1,2,\ldots,N}L^n_T$.

\begin{lemma}
For any $K$ arrays of $N$ probabilities
$p^1_{k},p^2_{k},\ldots,p^N_{k}$, $k=1,2,\ldots,K$, where $p^n_{k}\in
[0,1]$ for all $n=1,2,\ldots,N$ and $k=1,2,\ldots,K$ and
$\sum_{n=1}^Np^n_{k}=1$ for all $k$, there is $n$ such that
\begin{equation*}
\prod_{k=1}^Kp^n_k\le\frac{1}{N^K}\enspace.
\end{equation*}
\end{lemma}

\begin{proof}
Assume the converse. Let $\prod_{k=1}^Kp^n_k>1/N^K$ for all $n$. By
the inequality of arithmetic and geometric means
\begin{equation*}
\sum_{k=1}^K\frac{p^n_k}{K} \ge \left(\prod_{k=1}^K
p^n_k\right)^{\frac 1K} 
\end{equation*}
for all $n=1,2,\ldots,N$. Summing the left-hand side over $n$ yields
\begin{equation*}
\sum_{n=1}^N \sum_{k=1}^K\frac{p^n_k}{K} = \frac 1K \sum_{k=1}^K
\sum_{n=1}^N p^n_k =1\enspace.
\end{equation*}
Summing the right-hand side over $n$ and 
using the assumption on the products of $p^n_k$, we get
\begin{equation*}
\sum_{n=1}^N\left(\prod_{k=1}^K
p^n_k\right)^{\frac 1K} > \sum_{n=1}^N\left(\frac 1{N^K}\right)^{\frac
  1K} =\sum_{n=1}^N \frac 1N =1\enspace.
\end{equation*}
The contradiction proves the lemma.
\end{proof}

Here is the strategy for the nature. Upon getting the probability
distributions from the learner, it finds $n_0$ such that
$\prod_{k=1}^{K_t}p^{n_0}_{t,k}\le 1/N^{K_t}$ and sets
$\ell^{n_0}_{t,1}=\ell^{n_0}_{t,2}=\ldots=\ell^{n_0}_{t,K_t}=0$ and
$\ell^{n}_{t,k}=+\infty$ for all other $n$ and $k=1,2,\ldots,K_t$. The
learner suffers loss
\begin{equation*}
\ell_t = -\sum_{k=1}^{K_t}\ln p^{n_0}_{t,k}= -\ln\prod_{k=1}^{K_t}
p^{n_0}_{t,k}\ge -\ln\frac 1{N^{K_t}}=K_t\ln N
\end{equation*}
while $\ell^{n_0}_t=0$. We see that over a single pack of size $K$ we
can achieve the regret of $K\ln N$. Thus every upper bound of the form
\begin{equation*}
L_T \le L^n_T+R
\end{equation*}
should have $R\ge K_1\ln N$, where $K_1$ is the size of the first pack.

\section{Experiments}
\label{sect_experiments}

In this section, we present some empirical results. Our purpose is
twofold. First, we want to study the behaviour of the algorithms
described above in practice. Secondly, we want to demonstrate the
power of on-line learning.

\subsection{Datasets and Models}

For our experiments, we used two datasets of house prices. There is a
tradition of using house prices as a benchmark for machine learning
algorithms going back to the Boston housing dataset. However, batch
learning protocols have hitherto been used in most studies.

Recently extensive datasets with timestamps have become
available. They call for on-line learning protocols.  Property prices
are prone to strong movements over time and the pattern of change may
be complicated. On-line algorithms should capture these patterns.

\subsubsection{Ames House Prices}

The first dataset describes the property sales that occurred in Ames,
Iowa between 2006 and 2010. The dataset contains records of 2930 house
sales transactions with 80 attributes, which are a mixture of nominal,
ordinal, continuous, and discrete parameters (including physical
property measurements) affecting the property value.  The dataset was
compiled by Dean De Cock for use in statistics education
\cite{DeCock2011Ames} as a modern substitute for the Boston Housing
dataset.

There are timestamps in the dataset, but they contain only the month
and the year of the purchase. The date is not available.  Therefore,
one can not apply the on-line protocol directly to the problem as at
each time we observe a vector of outcomes instead of a single
outcome. It is natural to try and work out all predicted prices for a
particular month on the basis of past months and only then to see the
true prices. One month of transactions makes what we call a pack in
this paper. We interpret the problem as falling under
Protocol~\ref{prot_with_signals_pack}. The prediction and outcome
spaces are a real interval $\Omega = \Gamma = [A, B]$ and the square
loss function $\lambda(\gamma, \omega) = (\gamma - \omega)^2$ is used.
This game is mixable and the maximum $\eta=2/(B-A)^2$ is the maximum
such that $C_\eta=1$ (see \cite{vovk_cols}; the derivation for the
interval $[-Y,Y]$ can be easily adapted for $[A,B]$).

We apply AAP algorithms to Ames house prices data set. In the first
set of experiments our experts are linear regression models based on
only two attributes: the neighbourhood and the total square footage of
the dwelling. These simple models explain around 80\% of the variation
in sales prices and they are very easy to train. Each expert has been
trained on one month of the first year of the data. Hence there are 12
`monthly' experts.

In the second set of experiments on Ames house dataset we use random
forests (RF) models after~\cite{Bellotti2017CP_supplementary}. A model
was built for each quarter of the first year.  Hence there are four
`quarterly' experts. They take longer to train but produce better
results.  Note that `monthly' RF experts were not practical. Training
a tree requires a lot of data and `monthly' experts returned very poor
results.

We then apply the experts to predict the prices starting from year
two.

\subsubsection{London House Prices}

Another data set that was used to compare the performance of AAP
contained house prices in and around London over the period 2009 to
2014. This dataset was made publicly available by the Land Registry
in the UK and was originally sourced as part of a Kaggle
competition. The Property Price data consists of details for property
sales and contains around 1.38 million observations.  
This data set was studied before to provide reliable region 
predictions for Automated Valuation Models of house prices  \cite{Bellotti2017CP}.

As with Ames dataset, we use linear regression models that were built
for each month of the first year of the data as experts of AAP. 
Features that were used in regression models contain information about the property: 
property type, whether new build, whether free- or leasehold. 
Along with the information about the proximity to tubes and railways, models use
the English indices of deprivation 2010 which measures relative levels of deprivation.
The following deprivation scores were used in models: 
income, employment, health and disability, education for children and skills
for adults, barriers to housing and services with sub-domains wider barriers
and geographical barriers, crime, living environment score
with sub-domains for indoor and outdoor living (i.e. quality of housing and external environment, respectively). 
Additional to the general income score, separate
scores for income deprivation affecting children and the older population
were used.

In the second set of experiments on London house dataset we use RF models 
built for each month of the first year as experts. 
Compare to Ames dataset, London house dataset contains enough observations to 
train RF models on one month of the data. Hence we have 12 `monthly' experts.

\subsection{Comparison of Merging Algorithms}
\subsubsection{Comparison of AAP with Parallel Copies of AA}

We start by comparing the family of AAP merging algorithms against
parallel copies of AA. While for AAP algorithms the order of examples
in the pack makes no difference, for parallel copies it is
important. To analyse the dependency on the order we ran parallel
copies 500 times randomly shuffling each pack each time. 

Figure~\ref{fig1a} shows the histogram of total losses of the parallel
copies of AA with regression experts on Ames house dataset.  The
average total loss of parallel copies is almost the same as the total
losses of AAP-incremental and AAP-max. AAP-current shows the best
performance among AAP algorithms with a slight improvement over the
mean. While the performance of parallel copies {\em can} be better,
AAP family provides {\em stable} order-independent performance, which
is good on average.

There is one remarkable ordering where parallel copies show greatly
superior performance. If packs are ordered by PID (i.e., as in the
database), parallel copies suffer substantially lower loss. PID
(Parcel identification ID) is assigned to each property by the tax
assessor. It is related to the geographical location. When the packs
are ordered by PID, parallel copies benefit from geographical
proximity of the houses; each copy happens to get similar houses.

Figure~\ref{fig1b} shows the histogram of total losses of the
algorithm with parallel copies of AA with RF experts. In this case,
the average total loss of this algorithm is slightly lower than total
losses of AAP-incremental and AAP-max. AA with parallel copies ordered
by PID has lower total loss than the average. AAP-current has the
lowest total loss among the AAP family and even beats the parallel
copies for PID-ordered packs.



\begin{figure}[!h]
	\centering
	\subfloat[Regression on Ames house prices]{\label{fig1a}\includegraphics[width=60mm, height=60mm]{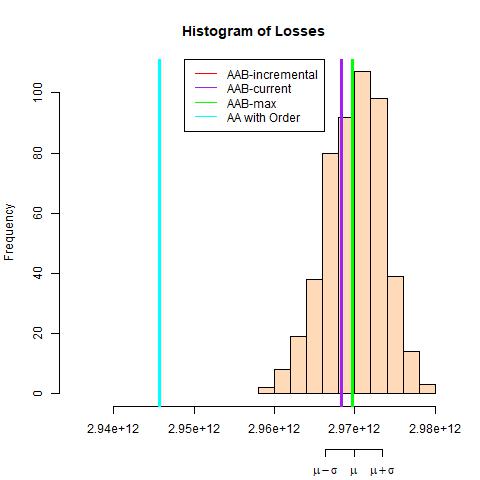}}
	\subfloat[RF on Ames house prices]{\label{fig1b}\includegraphics[width=60mm, height=60mm]{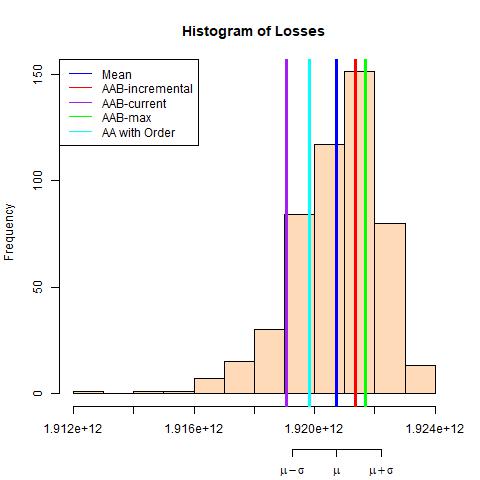}}
		\caption{Histogram of total losses}
	\label{fig1}
\end{figure}



\subsubsection{Comparison of AAP-incremental and AAP-max}

Figure~\ref{fig2a} illustrates the difference in total losses of
AAP-incremental and AAP-max on Ames house prices data with regression
models. AAP-incremental performs better at the beginning of the period
when the current maximum size of the pack is much lower than the maximum
pack of the whole period. After that, AAP-incremental and AAP-max
have almost similar performance and the total losses level out.

Figure~\ref{fig5a} illustrates the difference in total losses of
AAP-incremental and AAP-max on Ames house prices data with RF
experts. Figures \ref{fig6a}, \ref{fig7a} show the same experiment conducted on
London house prices. In these cases AAP-incremental steadily
outperforms AAP-max.

\begin{figure}[!h]
	\centering
	\subfloat[Regression on Ames house prices]{\label{fig2a}\includegraphics[width=50mm, height=50mm]{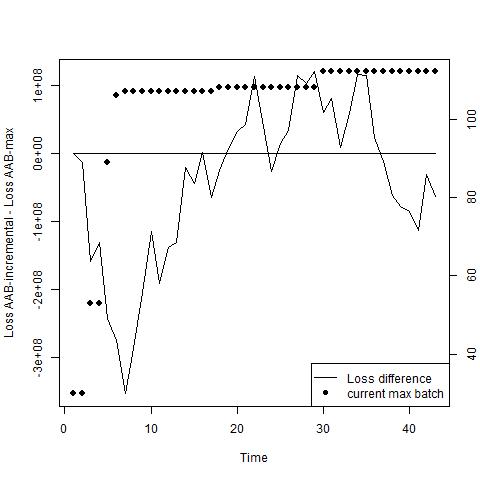}}
	\subfloat[RF on Ames house prices]{\label{fig5a}\includegraphics[width=50mm, height=50mm]{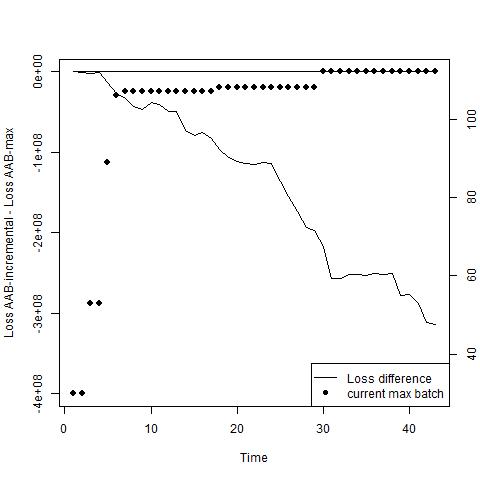}} \\
	\subfloat[Regression on London house prices]{\label{fig6a}\includegraphics[width=50mm, height=50mm]{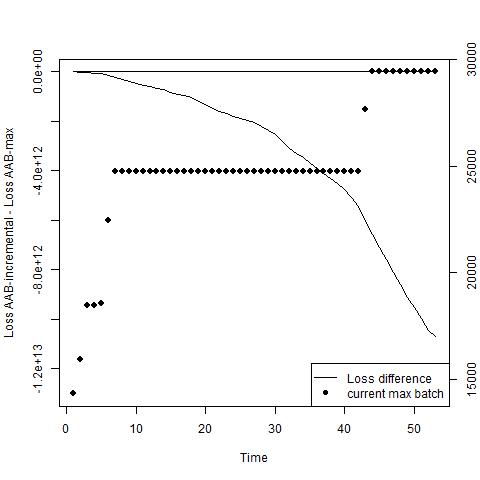}}
	\subfloat[RF on London house prices]{\label{fig7a}\includegraphics[width=50mm, height=50mm]{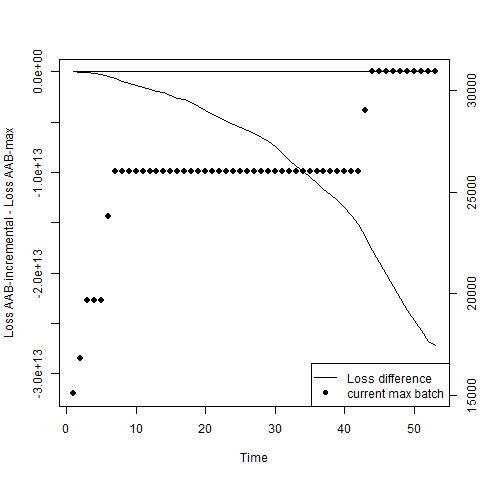}}
	\caption{Comparison of total losses of AAP-incremental and AAP-max}
	\label{fig2}
\end{figure}


\subsubsection{Comparison of AAP-current and AAP-incremental}

Figure \ref{fig3} illustrates the difference in total losses of
AAP-current and AAP-incremental. Figures \ref{fig2b} and \ref{fig5b}
show results for Ames house prices for regression and RF experts
respectively, Figures \ref{fig6b}, \ref{fig7b} ~--- for London house prices. 
In all experiments AAP-current steadily outperforms AAP-incremental.

The performance of AAP-current is remarkable because by design it is
not optimised to minimise the total loss. The bound of
Corollary~\ref{cor_current} is weak in comparison to that of
Theorem~\ref{theorem_incremental}. In a way, here we assess AAP-current
with a measure it is not good at. Still optimal decisions of
AAP-current produce superior performance.

\begin{figure}[!h]
	\centering
	\subfloat[Regression on Ames house prices]{\label{fig2b}\includegraphics[width=50mm, height=50mm]{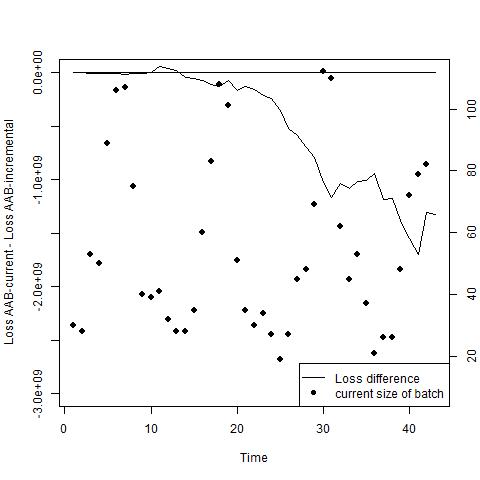}}
	\subfloat[RF on Ames house prices]{\label{fig5b}\includegraphics[width=50mm, height=50mm]{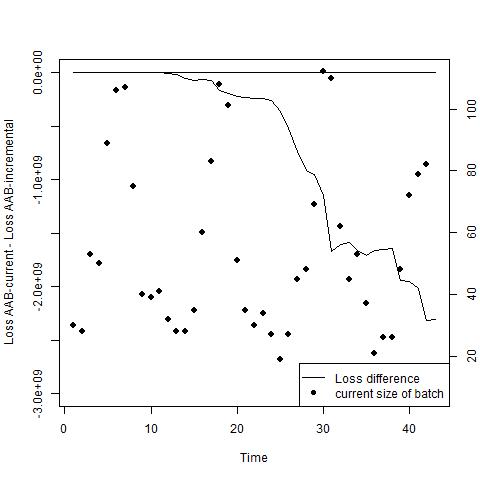}} \\
	\subfloat[Regression on London house prices]{\label{fig6b}\includegraphics[width=50mm, height=50mm]{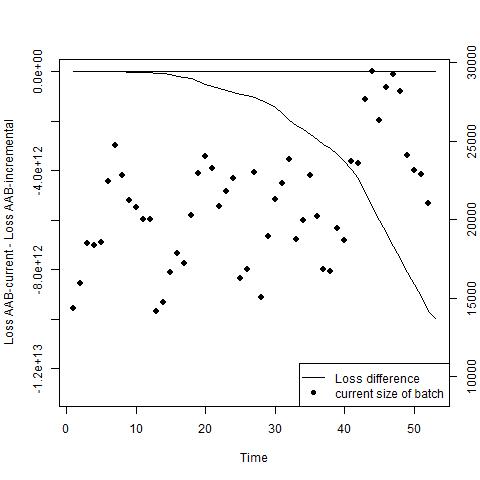}}
	\subfloat[RF on London house prices]{\label{fig7b}\includegraphics[width=50mm, height=50mm]{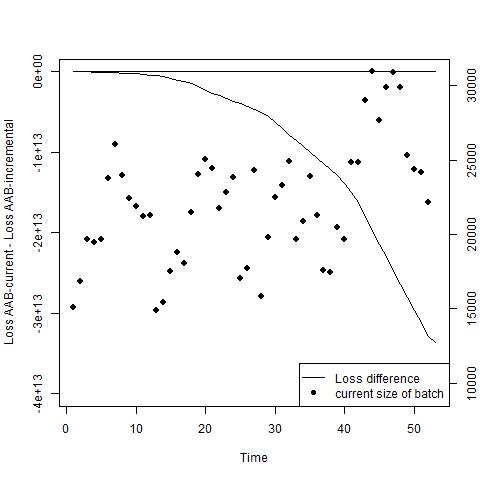}}
	\caption{Comparison of total losses of AAP-current and AAP-incremental}
	\label{fig3}
\end{figure}

\subsubsection{Comparison of AAP with Batch Models}

In this section we compare AAP-incremental with two straightforward
ways of prediction, which are essentially batch. One goal we have here
is to do a sanity check and verify whether we are not studying
properties of very bad algorithms. Secondly, we want to show that
prediction with expert advice may show better ways of handling
the available historical information.

In AAP we use linear regression models that have been trained on each
month of the first year of the data. Is the performance of these
models affected by straightforward seasonality?  What if we always
predicts January with the January model, February with the February
model etc?

The first batch model we compare our on-line algorithm to is the seasonal model
that predict January with linear regression model that has been
trained on January of the first year, February~--- with linear model
of February of the first year, etc. 

In the case of `quarterly' RF experts, we compete with seasonal model
that predict first quarter with RF model that has been trained on the first quarter, 
second quarter~--- with RF of the second quarter, etc.

Secondly, what if we train a model on the whole of the first year?
This may be more expensive than training smaller models, but what do
we gain in performance? The second batch model is the linear model
that has been trained on the whole first year of the
data. In case of RF experts, we compete with RF model that has been trained 
on the first year of the data.

Figure~\ref{fig4} shows the comparison of total losses of AAP-current
and batch linear regression models for Ames house dataset. AAP-current consistently performs better than
the seasonal batch model. Thus the straightforward utilisation of
seasonality does not help.

When compared to the linear regression model of the first year, AAP-current initially
has higher losses but it becomes better towards the end. It could be
explained as follows. AAP-current needs time to train until it becomes
good in prediction. These results show that we can make a better use of the past data with
prediction with expert advice than with models that were trained in
the batch mode.

Table \ref{Losses} shows total losses of algorithm (divided by
$10^{12}$). AAP algorithms always outperform seasonal batch models.
As compared to linear regression batch models that were built on the
first year of the data, AAP is slightly better on Ames house dataset
and slightly worse on London house dataset. RF batch models that were
built on the first year of the datasets constantly outperform AAP
algorithms.

The losses quoted for the parallel copies are the means over 500
random shuffles as explained above. The experiment was not run for
London house prices as it is very time-consuming.
\begin{figure}[!h]
	\subfloat[Loss difference of AAP-current and \newline monthly batch]{\label{fig3a}\includegraphics[width=60mm, height=60mm]{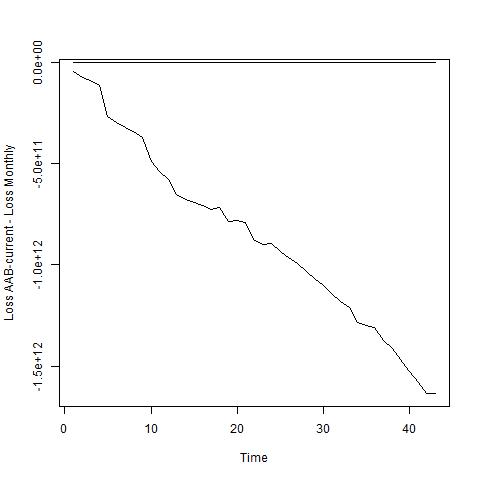}}
	\subfloat[Loss difference of AAP-current and \newline year batch]{\label{fig3b}\includegraphics[width=60mm, height=60mm]{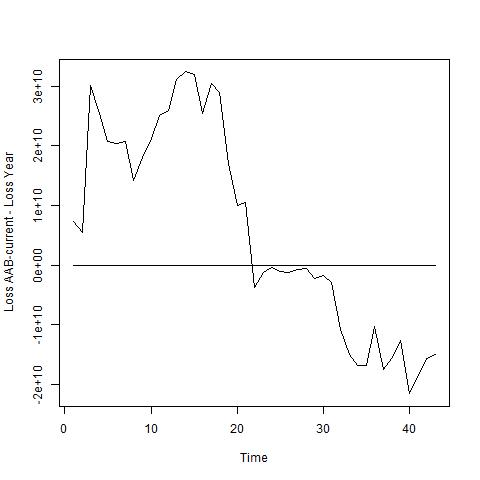}}
	\caption{Comparison of total losses of AAP and batch models}
	\label{fig4}
\end{figure}

 \begin{table}[H]
\caption{Total losses}
   \begin{tabular}[!h]{|l|l|l|l|l|}
	\hline
	& Ames reg & Ames RF & London reg & London RF \\
	\hline
	AAP max & 2.9698 & 1.9217 & 28484.34 & 18819.01\\
	AAP incremental & 2.9697 & 1.9214 & 28474.26 & 18791.79\\
	AAP current & 2.9684 & 1.9191 & 28461.39 & 18758.14 \\
	Parallel copies & 2.9699 & 1.9207 & - & -  \\
	Batch Seasonal & 4.6036 & 2.1485 &  28750.21 & 22543.43\\
	Batch Year & 2.9833 &  1.4699 & 28272.74 & 15877.6 \\
	\hline
\end{tabular}

\label{Losses}

\subsection{Conclusion}
\end{table}
We tested the performance of AAP against the algorithm with parallel
copies of AA. We found that the average performance of algorithm with
parallel copies of AA is close to the performance of AAP. AA with
parallel copies ordered by PID has lower total loss than the average
of AA with parallel copies which means that a meaningful ordering can
have a big impact on the performance of the algorithm. In the absence
of such knowledge, AAP algorithms provide more stable performance.

AAP-current is constantly outperforming AAP-incremental and AAP-max on
two data sets. Therefore, we do not need to know the maximum size of
the pack in advance.

Also experiments showed that in some cases we could get the better use
of the past data with AAP than with models that were trained in the
batch mode.

\bibliographystyle{alpha}
\bibliography{C:/Bib/yk_bibfile}
\end{document}